\documentclass{llncs}
\usepackage{amsmath}
\usepackage{times}
\usepackage{indentfirst}
\usepackage{graphicx}

\begin{document}
\title{Some characteristics of matroids through rough sets }         

\author{Lirun Su, William Zhu}
\institute{
Lab of Granular Computing, \\
Zhangzhou Normal University, Zhangzhou, China}



\date{\today}          
\maketitle

\large
\begin{abstract}

At present, practical application and theoretical discussion of rough sets are two hot problems in computer science. The core concepts of rough set theory are upper and lower approximation operators based on equivalence relations. Matroid, as a branch of mathematics, is a structure that generalizes linear independence in vector spaces. Further, matroid theory borrows extensively from the terminology of linear algebra and graph theory. We can combine rough set theory with matroid theory through using rough sets to study some characteristics of matroids. In this paper, we apply rough sets to matroids through defining a family of sets which are constructed from the upper approximation operator with respect to an equivalence relation.
First, we prove the family of sets satisfies the support set axioms of matroids, and then we obtain a matroid. We say the matroids induced by the equivalence relation and a type of matroid, namely support matroid, is induced.
Second, through rough sets, some characteristics of matroids such as independent sets, support sets, bases, hyperplanes and closed sets are investigated.\\



\end{abstract}

\section{Introduction}
 With the advent of huge data, knowledge analysis and disposal technology become increasingly important. It is difficult to extract useful information from vague and incomplete data. In order to deal with this issue, many scholars have put forward various useful methods. As one of those important techniques, rough set theory was proposed by Pawlak\cite{Pawlak82Rough} in 1982 to deal with uncertainty, incompleteness and vagueness. Because of its advantage of not depending on priori knowledge, it attracted much research interest in the past years. In application, rough set theory has already been applied to various fields such as process control\cite{YamaguchiLiNagai07AGrey-based}, economics, medical diagnosis\cite{TayShen02Economic} and attribute reduction\cite{MinHeQianZhu11Test}. In theory, classical rough sets are based on equivalence relations. They have been extended to fuzzy rough sets\cite{SunGongChen08Fuzzy,GongSunChen08Rough}, relation-based rough sets\cite{Kryszkiewicz98Rule,MengShi09AFast} and covering-based rough sets\cite{TangSheZhu11Covering-based,DengChenXuDai07ANovel,Zhu07AClass,Drwal00Rough,Pawlak05Rough}.

 Matroids\cite{Lai01Matroid,LiuChen94Matroid} were proposed by Whitney in 1935 to denote a class of fundamental objects arising from matrices in a certain way. They borrow extensively from linear algebra and graph theory, and made great progress in recent decades. In theory, matroids are connected with covering-based rough sets\cite{WangZhu11Matroidal,WangZhuMin11Transversal,WangZhuZhuMin12matroidalstructure}, generalized rough sets\cite{WangZhuMin11TheVectorially} and fuzzy sets\cite{RoyVoxman88Fuzzy,RoyVoxman92FuzzyMatroid,Li07Some,RoyVoxman89Bases} through some constructive methods\cite{Yao98Constructive,Zhu09RelationshipBetween}. In application, matroids have been used in diverse fields such as algorithms of attribute reduction\cite{YaoZhao08Attribute} and combinatorial optimization\cite{Lawler01Combinatorialoptimization}.

In this paper,  a matroidal structure of rough sets is constructed, and then some characteristics of the matroid are studied through rough sets. First, for an equivalence relation on a universe, we define a family of subsets of the universe through the upper approximation operator, and prove it satisfies the support set axioms of matroids. A matroid is generated by the family of subsets, and we say the matroid is induced by the equivalence relation and a type of matroid, namely support matroid, is defined. In this way, we bridge matroids and rough sets through support sets in matroids, and study the relationships between rough sets and matroids. Second, Based on the matroid, we study the relationships among upper approximations, equivalence classes and some concepts in  matroids. For example, this paper uses upper approximations and equivalence classes to describe  bases, hyperplanes, independent sets and closed sets, respectively. Furthermore, we investigate some necessary and sufficient conditions of closed sets from the viewpoint of rough sets.

 The rest of this paper is organized as follows: In Section~\ref{sc2}, we review some basic definitions of rough sets and matroids. Section~\ref{sc3} introduces the matroids induced by equivalence relations and studies the characteristics of the matroids through rough sets. Finally, we conclude this paper in Section~\ref{sc4}.

\section{Background}\label{sc2}
In this section, we review some fundamental definitions of Pawlak's rough sets and matroids.
\subsection{Fundamentals of Pawlak's rough sets}
\label{S:Definitions of rough matroids}

In this subsection, we recall some basic concepts of rough sets. Let $U$ be a finite and nonempty set called a universe. Let $R$ be an equivalence relation on $U$, i.e., $R$ is reflexive, transitive and symmetric. A universe together with an equivalence relation on the universe forms an approximation space.

\begin{definition}(Approximation space\cite{Pawlak82Rough})
Let $U$ be a finite and nonempty universe and $R$ an equivalence relation on $U$. The ordered pair $(U, R)$ is called an approximation space.

\end{definition}

In rough sets, we use a pair of approximation operators to describe an object of $U$. In the following definition, the pair of approximation operators are introduced.

\begin{definition}\label{D:2}(Approximation operators\cite{Yao98Constructive})
Let $R$ be an equivalence relation on $U$. A pair of approximation operators $R^{*}, R_{*}: 2^{U}\longrightarrow 2^{U}$ are defined as follows: for all $X\subseteq U$,

$R_{*}(X) = \{x\in U|RN(x)\subseteq X\}$,

$R^{*}(X) = \{x\in U|RN(x)\cap X\neq \emptyset\}$,

\end{definition}
where $RN(x) = \{y\in U|xRy\}$. They are called the lower and upper approximation operators with respect to $R$, respectively.

 $\sim X$ denotes the complement of $X$ in $U$ and $Y\subseteq U$. We have the following properties of rough
sets:\\
(1H) $R^{*}(U) = U$ (Co-normality)\\
(1L) $R_{*}(X) = \emptyset$ (Normality)\\
(2L) $R_{*}(X)\subseteq X$ (Contraction)\\
(2H) $X\subseteq R^{*}(X)$ (Extension)\\
(3L) $R_{*}(X\cap Y) =R_{*}(X)\cap R_{*}(X)$ (Multiplication)\\
(3H) $R^{*}(X\cup Y) = R^{*}(X)\cup R^{*}(Y)$ (Addition)\\
(4L) $R_{*}(X) = \sim R^{*}(\sim X)$ (Duality)\\
(4H) $R^{*}(X) = \sim R_{*}(\sim X)$ (Duality)\\
(5L) $R_{*}(\sim R_{*}(X)) = \sim R_{*}(X)$( Lower-complement relation)\\
(5H) $R^{*}(\sim R^{*}(X)) = \sim R^{*}(X)$ (Upper-complement relation)\\
(6H) $X\subseteq Y\Rightarrow R^{*}(X)\subseteq R^{*}(Y)$ (Monotone)

The(2L), (2H), (4L), (4H), (5L) and (5H) are characteristic properties of the lower and upper approximation operators, respectively. In other words, all other properties can be deduced from these properties\cite{ZhuHe00Logic,ZhuHe00TheAxiomization,LinLiu94Rough}.

 In an approximation space, a set is called a precise set if it can be precisely described by the equivalence relation; otherwise, it is called a rough set.

\begin{definition}\label{D:jque}($R$-precise and $R$-rough set\cite{Yao98Constructive})
Let $R$ be an equivalence relation on $U$. For all $X\subseteq U$, if $R^{*}(X)=R_{*}(X)$, then we say $X$ is a $R$-precise set; otherwise, we say $X$ is a $R$-rough set.
\end{definition}
\subsection{Fundamentals of matroids }
\label{Fofmatroid}

 Matroids were established as a generalization or a connection, of graph theory and linear algebra. In this subsection, some concepts of matroids such as independent sets, support sets, bases, rank function, closed sets and closure will be introduced.
\begin{definition}(Matroid\cite{Lai01Matroid,LiuChen94Matroid})
A matroid $M$ is an ordered pair $(U, \mathbf{I})$, where $U$ (the ground set) is a finite set, and $\mathbf{I}$ (the independent sets) a family of subsets of $U$  with the following properties:\\
(I1) $\emptyset\in \mathbf{I}$;\\
(I2) If $I\in \mathbf{I}, I^{'}\subseteq I$, then $I^{'}\in \mathbf{I}$;\\
(I3) If $I_{1}, I_{2}\in \mathbf{I}$ and $|I_{1}| < |I_{2}|$, then there exists $e\in I_{2}-I_{1}$ such that $I_{1}\bigcup \{e\}\in \mathbf{I}$,
where $|I|$ denotes the cardinality of $I$.
\end{definition}

\begin{example}
\label{li1}
Let $U = \{a_{1}, a_{2}, a_{3}, a_{4}\}$ where $a_{1} = \{1, 0, 1\}^{T}$, $a_{2} = \{0, 1, 0\}^{T}$, $a_{3} = \{-1, 0, 1\}^{T}$, $a_{4} = \{0, 0, 1\}^{T}$. That $\mathbf{I} = \{\emptyset, \{a_{1}\}, \{a_{2}\}, \{a_{3}\}, \{a_{4}\}, $ $\{a_{1}, a_{2}\}, \{a_{1}, a_{3}\}, \{a_{2}, a_{3}\}$, $\{a_{2}, a_{4}\}, $ $\{a_{1}, a_{4}\}, \{a_{3}, a_{4}\}, \{a_{1}, a_{2}, a_{3}\}, \{a_{1},$ $ a_{2}, a_{4}\}, \{a_{2}, a_{3}, a_{4}\}\}$. $M = (U, \mathbf{I})$ is a matroid.
\end{example}
The above example shows that the independent sets of a matroid is a generalization of the linearly independent sets. Similarly, the maximal independent sets are generalized to the bases of matroids.
\begin{definition}\label{D:5}
Let $\mathbf{A}\subseteq 2^{U}$  be a family of subsets of $U$. One can denote\\
$Upp(\mathbf{A})= \{X\subseteq U:\exists A\in \mathbf{A}$ s.t.  $A\subseteq X\}$,\\
$Low(\mathbf{A})= \{X\subseteq U:\exists A\in \mathbf{A}$ s.t.  $X\subseteq A\}$,\\
$Max(\mathbf{A})= \{X\in \mathbf{A}:\forall Y\in \mathbf{A}, X\subseteq Y\Rightarrow X = Y\}$,\\
$Min(\mathbf{A})= \{X\in \mathbf{A}:\forall Y\in \mathbf{A}, Y\subseteq X\Rightarrow X = Y\}$,\\
$Opp(\mathbf{A})= \{X\subseteq U:X\notin \mathbf{A}\}$.
\end{definition}
\begin{definition}(Base\cite{Lai01Matroid,LiuChen94Matroid})
\label{ji}
Let $M = (U, \mathbf{I})$ be a matroid. A maximal independent set in $M$ is called a base of $M$, and we denote the family of all bases of $M$ by $\mathbf{B}(M)$, i.e., $\mathbf{B}(M) = Max(\mathbf{I})$.
\end{definition}

 The dimension of a vector space and the rank of a matrix are quite useful concepts in linear algebra. It is necessary to extend these two concepts to matroids.

\begin{definition}(Rank function\cite{Lai01Matroid,LiuChen94Matroid})
\label{D:zhi}
 Let $M = (U, \mathbf{I})$ be a matroid. The rank function $r_{M}$ of $M$ is defined as follows: for all $X \subseteq U$,
\begin{center}
$ r_{M}(X) = max\{|I|| I\subseteq X, I\in \mathbf{I}\}$.
\end{center}
\end{definition}

In graph theory, all acyclic subgraphs are spanning subgraphs. This concept can be extended  to matroid theory, and a new concept called support set can be obtained.
\begin{definition}(Support set\cite{Lai01Matroid,LiuChen94Matroid})
Let $M = (U, \mathbf{I})$ be a matroid. For all $X\subseteq U$, if there exists a base $B\in\mathbf{B}(M)$ such that $B \subseteq X$, then $X$ is called a support set of $M$, and we denote the family of all support sets of $M$ by $\mathbf{S}(M)$.
\end{definition}

Based on the rank function of a matroid,  the closure operator which reflects the dependency between a set and elements can be defined.
\begin{definition}\label{D:9}(Closure\cite{Lai01Matroid,LiuChen94Matroid})
Let $M = (U, \mathbf{I})$ be a matroid. For all $X\subseteq U$, the closure operator $cl_{M}$ of $M$ is defined as $cl_{M}(X)=\{e\in U| r_{M}(X) = r_{M}(X\cup \{e\})\}$ . $cl_{M}(X)$ is called the closure of $X$ in $M$.
\end{definition}
\begin{definition}\label{D:10}(Closed set\cite{Lai01Matroid,LiuChen94Matroid})
Let $M = (U, \mathbf{I})$ be a matroid and $X$ a subset of universe. $X$ is called a closed set of $M$ if $cl_{M}(X) = X$.
\end{definition}

Hyperplane is a significant concept in matriods. In this paper, we combine it with the upper approximation operator of rough sets, and we study some characteristics of hyperplane through rough sets.
\begin{definition}(Hyperplane\cite{Lai01Matroid,LiuChen94Matroid})\label{Hyperplane}
Let $M = (U, \mathbf{I})$ be a matroid. For all $H\subseteq U$, if $H$ is a closed set and $r_{M}(H)=r_{M}(U)-1$, then $H$ is called a hyperplane of $M$, and we denote the family of all hyperplanes of $M$ by $\mathbf{H}(M)$.
\end{definition}
The above definitions show the relationships among matroid theory, graph theory and linear algebra. The following proposition indicates a matroid can be defined from the viewpoint of support set.
\begin{proposition}(Support set axioms\cite{Lai01Matroid,LiuChen94Matroid})
\label{P:Supportsetaxioms}
Let $\mathbf{S}$ be a family of subsets of $U$. Then there exists $M = (U, \mathbf{I})$ such that $\mathbf{S} = \mathbf{S}(M)$ if and only if $\mathbf{S}$ satisfies the following three conditions:\\(S1) $\mathbf{S}$ contain a subset at least;\\
(S2) If $S_{1}\in \mathbf{S}$, and $S_{1}\subseteq S_{2}$, then $S_{2}\in \mathbf{S}$;\\
(S3) If $S_{1}, S_{2}\in \mathbf{S}, |S_{1}|>|S_{2}|$, then there exists $e\in S_{1}-S_{2}$ such that $S_{1}-\{e\}\in \mathbf{S}$.
\end{proposition}
\begin{example}(Continued from Example \ref{li1}) The family of support sets of $M$ is
$\mathbf{S}(M) = \{\{a_{1}, a_{2}, a_{3}\}, \{a_{1}, a_{2}, a_{4}\},$ $ \{a_{2}, a_{3}, a_{4}\}$, $\{a_{1}, a_{2}, $ $a_{3}, a_{4}\}\}$.
\end{example}
\section{Matroid induced by an equivalence relation}\label{sc3}
A matroid can be defined from different viewpoints such as independent sets and support sets. In this section, we will induce a matroid by an equivalence relation. We construct a family $\mathbf{S}(R)$ by the upper approximation operator, and prove that $\mathbf{S}(R)$ satisfies the support set axioms of matroids. Therefore, $\mathbf{S}(R)$ can uniquely determine a matoid, which is denoted by $M(R)$.
\begin{definition}
\label{D13}
Let $R$ be an equivalence relation on $U$. We define a family of subsets of $U$ as follows:
\begin{center}
$\mathbf{S}(R) = \{X\subseteq U|R^{*}(X) = U\}$.
\end{center}
\end{definition}
In fact, $\mathbf{S}(R)$ satisfies the support set axioms. In other words, it uniquely determines a matroid.
\begin{proposition}\label{P:14}
Let $R$ be an equivalence relation on $U$. Then $\mathbf{S}(R) = \{X\subseteq U|\forall x\in U,|RN(x)\cap X|\geq 1\}$.
\end{proposition}
\begin{proof}
  If $X\in\{X\subseteq U|R^{*}(X) = U\}$, according to Definition~\ref{D:2}, then for all $x, RN(x)\cap X\neq \emptyset$. Suppose that for all $x$, $|RN(x)\cap X|<1$. Then there exists $x_{1}\in U$ such that $RN(x_{1})\cap X =\emptyset$, which is contradictory to $RN(x)\cap X\neq \emptyset$ for all $x$. Hence $|RN(x)\cap X|\geq 1$ for all $x$. So we get that $\{X\subseteq U|R^{*}(X) = U\}\subseteq \{X\subseteq U|\forall x\in U,|RN(x)\cap X|\geq 1\}$. Conversely, for all $X\in\{X\subseteq U|\forall x\in U,|RN(x)\cap X|\geq 1\}$, this implies that $RN(x)\cap X\neq \emptyset$. According to Definition~\ref{D:2}, we have that $X\in\{X\subseteq U|R^{*}(X) = U\}$. Therefore, it is clear that $\{X\subseteq U|\forall x\in U,|RN(x)\cap X|\geq 1\}\subseteq \{X\subseteq U|R^{*}(X) = U\}$. This completes the proof.
\end{proof}
\begin{proposition}
Let $R$ be an equivalence relation on $U$. Then $ \mathbf{S}(R)$ satisfies (S1), (S2) and (S3) of Proposition~\ref{P:Supportsetaxioms}.
\end{proposition}
\begin{proof}
 According to Definition~\ref{D:2} and (6H) in section~\ref{sc2}, it is obvious that $\mathbf{S}(R)$ satisfies (S1) and (S2) . We need to prove only that $\mathbf{S}(R)$ satisfies (S3).  Suppose that $S_{1}, S_{2}\in \mathbf{S}(R)$ and $|S_{1}|>|S_{2}|$. According to Definition~\ref{D13} and Proposition~\ref{P:14}, let $U/R = \{RN(x_{1}),$ $ RN(x_{2}),...,$ $RN(x_{k})\}$, where $RN(x_{i}) = \{y\in U|x_{i}Ry\}(1\leq i\leq k)$. For all $x_{i}\in U(1\leq i\leq k)$, we have that $\mid RN(x_{i})\cap S_{1}\mid\geq 1$ and $\mid RN(x_{i})\cap S_{2}\mid\geq 1$. Suppose that for all $x_{i}\in U(1\leq i\leq k)$, $|RN(x_{i})\cap S_{1}|\leq |RN(x_{i})\cap S_{2}|$. Since $RN(x_{i})\cap RN(x_{j}) = \emptyset~(1\leq i\neq j\leq k)$, Thus $|\bigcup_{i=1}^{k}RN(x_{i})\cap S_{1}|\leq |\bigcup_{i=1}^{k}RN(x_{i})\cap S_{2}|$, i.e., $|S_{1}|\leq|S_{2}|$, which is contradictory to $|S_{1}|>|S_{2}|$. Hence there exists $x_{m}\in U (1\leq m\leq k)$ such that $\mid RN(x_{m})\cap S_{1}|>\mid RN(x_{m})\cap S_{2}\mid\geq 1$, which implies that there exists $e\in ((RN(x_{m})\cap S_{1})-(RN(x_{m})\cap S_{2}))\subseteq (RN(x_{m})\cap (S_{1}-S_{2}))\subseteq S_{1}-S_{2}$ such that $\forall x_{i}\in U(1\leq i\leq k), RN(x_{i})\cap (S_{1}-\{e\})\neq\emptyset$. According to Definition~\ref{D:2}, we have that $R^{*}(S_{1}-\{e\}) = U$, which implies that $S_{1}-\{e\}\in \mathbf{S}(R)$. This completes the proof.
\end{proof}

According to Proposition~\ref{P:Supportsetaxioms}, there exists a matroid on the universe such that $\mathbf{S}(R)$ is the family of its support sets. In fact, in literature\cite{WangZhuZhuMin12matroidalstructure}, a matroid is induced by an equivalence relation through the circuit axioms, and some concepts of matroid have been investigated. in this paper, through support axioms, a matroidal structure of an equivalence relation is established from a new perspective. In order to investigate the relationship beween matroids and rough sets. a type of matroid called support matroid is defined.
\begin{definition}(Support matroid)
Let $R$ be an equivalence relationon on $U$. The matroid
whose  the family of support sets is $\mathbf{S}(R)$ is denoted by $M(R) = (U, \mathbf{I}(R))$. We say $M(R) = (U, \mathbf{I}(R)$ is a support matroid induced by $R$, where $\mathbf{I}(R) = Low(Min(\mathbf{S}(R)))$.
\end{definition}

According to the duality of the lower and upper approximations, we can obtain the following proposition.
\begin{proposition}
Let $R$ be an equivalence relation on $U$ and $M(R)$ the support matroid induced by $R$. Then $\mathbf{S}(R) = \{X\subseteq U|R_{*}(\sim X) = \emptyset\}$.
\end{proposition}
\begin{proof}
According to Definition~\ref{D13} and (4H) in Section~\ref{sc2}, we obtain that $\mathbf{S}(R) = \{X\subseteq U|R_{*}(\sim X) = \emptyset\}$.
\end{proof}
The matroid $M(R)$ induced by an equivalence relation $R$ can be characterized from the viewpoint of rough sets. In the following, we will show how to describe some concepts of $M(R)$ through rough sets.
\begin{lemma}\cite{Lai01Matroid,LiuChen94Matroid}
Let $M = (U, \mathbf{I})$ be a matroid. Then $\mathbf{B}(M)$ = $Min(\mathbf{S}$ $(M))$.
\end{lemma}
The family of bases of matroid $M(R)$ can be expressed by equivalence classes induced by $R$.
\begin{proposition}
Let $R$ be an equivalence relation on $U$ and $M(R)$ the support matroid induced by $R$. Then
$\mathbf{B}(R) = \{X\subseteq U|\forall x\in U,|RN(x)\cap X| = 1\}$.
\end{proposition}
\begin{proof}
We need to prove only that $Min(\mathbf{S}(R)) = \{X\subseteq U|x\in U,|RN(x)$ $\cap X| = 1\}$. For all $X\in Min(\mathbf{S}(R))$, if there exists $x'\in U$ such that $|RN(x')\cap X|>1$, then there exists $e\in RN(x')\cap X$ and $e\neq x'$. Suppose that $X_{1} = X-\{e\}$. According to (S3) of Proposition~\ref{P:Supportsetaxioms}, we obtain that $X_{1}\in\mathbf{S}(R)$ and $X_{1}\subset X$, which is contradictory to $X\in Min(\mathbf{S}(R))$. Hence we have that $Min(\mathbf{S}(R))\subseteq \{X|x\in U, |RN(x)\cap X| = 1\}$. Conversely, for all $X\in \{X|x\in U, |RN(x)\cap X| = 1\}$, according to Definition~\ref{D:2} and Definition~\ref{D13}, it is clear that $R^{*}(X) = U$ and $X\in\mathbf{S}(R)$. For all $Y\subset X$, there exists $x\in U$ such that $|RN(x)\cap Y|=0$. Hence $R^{*}(Y)\neq U$. So $X\in Min(S(R))$. We obtain that $\{X|x\in U, |RN(x)\cap X| = 1\}\subseteq Min(\mathbf{S}(R))$. This completes the proof.
\end{proof}
\begin{lemma}\cite{Lai01Matroid,LiuChen94Matroid}\label{lm2}
Let $M = (U, \mathbf{I})$ be a matroid. Then $\mathbf{I}(M) = Low(Mi$ $n(\mathbf{S}(M)))$.
\end{lemma}
In linear space, independent sets express all linear independence groups. The following theorem shows independent sets of $M(R)$ can be described by equivalence class.
\begin{theorem}\label{theo1}
Let $R$ be an equivalence relation on $U$ and $M(R)$ the support matroid induced by $R$. Then
\begin{center}
$\mathbf{I}(R) = \{X\subseteq U| \forall x\in U, |RN(x)\cap X|\leq 1\}$.
\end{center}
\end{theorem}
\begin{proof}
According to Proposition~\ref{P:14} and Definition~\ref{D:5}, $ Low(Min(\mathbf{S}(R))) = Low(\{X\subseteq U|x\in U, |RN(x)\cap X| = 1\}) = \{X\subseteq U| \forall x\in U, |RN(x)\cap X|\leq 1\}$.
\end{proof}
\begin{proposition}\label{P6}
Let $R$ be an equivalence relation on $U$ and $M(R)$ the support matroid induced by $R$. Then
\begin{center}
$r_{(R)}(X)=|\{RN(x)|x\in U, RN(x)\cap X\neq\emptyset\}|$.
\end{center}
\end{proposition}
\begin{proof}
According to Theorem~\ref{theo1}, it is straightforward that $|I| = |\{RN(x)|$ $x\in U, RN(x)\cap I\neq\emptyset\}|$. According to Definition~\ref{D:zhi}, we get that $r_{(R)}(X)=|\{RN(x)|x\in U, RN(x)\cap X\neq\emptyset\}|$.
\end{proof}
\begin{lemma}\cite{Lai01Matroid,LiuChen94Matroid}
Let $M = (U, \mathbf{I})$ be a matroid. Then $\mathbf{H}(M) = Max(O$ $pp$ $(\mathbf{S}(M)))$.
\end{lemma}
\begin{proposition}
Let $R$ be an equivalence relation on $U$ and $M(R)$ the support matroid induced by $R$. Then
\begin{center}
$\mathbf{H}(R)=\{X\subseteq U|\forall x\notin X,R^{*}(X)=U-RN(x)\}$.
\end{center}
\end{proposition}
\begin{proof}
   It is obvious that $Max(Opp(\mathbf{S}(R))) = Max(\{X\subseteq U|R^{*}(X)\subset U\})$. According to Definition~\ref{D13}, for all $X\in Max(\{X\subseteq U|R^{*}(X)\subset U\})$,  there exists $x\in U$ such that $RN(x)\cap X = \emptyset$ and $R^{*}(X) = U-RN(x)$. Hence $X\in \{X\subseteq U|\forall x\notin X, R^{*}(X) = U-RN(x)\}$. Therefore, we have that $Max(Opp(S(R)))\subseteq \{X\subseteq U|\forall x\notin X, R^{*}(X) = U-RN(x)\}$. Conversely, for all $X\in \{X\subseteq U|\forall x\notin X, R^{*}(X) = U-RN(x)\}$, according to Definition~\ref{D:5}, this implies that $X\in Opp(\mathbf{S}(R))$. Together with $R^{*}(X)=U-RN(x)$, this means that $X\in Max(Opp(\mathbf{S}($ $R)))$. Hence $\{X\subseteq U|\forall x\notin X, R^{*}(X) = U-RN(x)\} \subseteq Max(Opp(\mathbf{S}(R)$ $))$. This completes the proof.
\end{proof}

Let $R$ be an equivalence relation on $U$. Suppose that $U/R = \{RN(x_{1}),$ $ RN(x_{2}),...,$ $RN(x_{k})\}$, where $RN(x_{i}) = \{y\in U|x_{i}Ry\}(1\leq i\leq k)$. The following proposition provides the necessary and sufficient condition when a subset is a closed set.
\begin{proposition}\label{P:bijiiff}
Let $R$ be an equivalence relation on $U$. For all $X\subseteq U$, $X$ is a closed set of $M(R)$ if and only if $X$ is a union of some elements of $U/R$.
\end{proposition}
\begin{proof}
 Suppose that $X$ is a closed set of $M(R)$. According to Definition~\ref{D:9}, we have that $cl_{(R)}(X)$ $=\{e\in U|r_{(R)}(X) = r_{(R)}(X\cup \{e\})\} = X$. Suppose $X$ is not a union of some elements of $U/R$. Then there exists $m (1\leq m\leq k)$ such that $RN(x_{m})\cap X\neq\emptyset$, and there exists $y\in RN(x_{m})$ such that $y\notin X$. According to Proposition~\ref{P6}, it is clear that $r_{(R)}(X) = r_{(R)}(X\cup \{y\})$. Hence $y\in cl_{(R)}(X) = X$, which is contradictory to $y\notin X$. Therefore, $X$ is a union of some elements of $U/R$. Conversely, suppose that $X$ is a union of some elements of $U/R$. On one hand,  according to Definition~\ref{D:9} and Proposition~\ref{P6}, we have that if $x\notin X$ then $x\notin cl_{(R)}(X)$ for all $x\in U$. Hence $cl_{(R)}(X)\subseteq X$. On the other hand, it is straightforward that $X\subseteq cl_{(R)}(X)$. Therefore, $cl_{(R)}(X) = X$, namely, $X$ is a closed set of $M(R)$. This completes the proof.
\end{proof}

Moreover, any closed set can be expressed by the lower and upper approximation operator, $R$-precise sets, respectively.

According to Proposition~\ref{P:bijiiff} and Definition~\ref{D:2}, we have the following corollary.
\begin{corollary}\label{C:1}
$X$ is a closed set of $M(R)$ if and only if $R^{*}(X) =X$.
\end{corollary}
\begin{corollary}\label{C:2}
$X$ is a closed set of $M(R)$ if and only if $R_{*}(X) = X$.
\end{corollary}

\begin{corollary}
$X$ is a closed set of $M(R)$ if and only if $R^{*}(X) = R_{*}(X)$.
\end{corollary}

According to Proposition~\ref{P:bijiiff} and Definition~\ref{D:jque}, we obtain the following result.
\begin{corollary}
$X$ is a closed set of $M(R)$ if and only if $X$ is a R-precise set.
\end{corollary}

Similarly, the rough set can be represented by the closed set of the matroid. In fact, a subset of a universe is a rough set if and only if it is not a closed set of the matroid.
\begin{corollary}
$X$ is not a closed set of $M(R)$ if and only if $X$ is a R-rough set.
\end{corollary}
\begin{lemma}\cite{Lai01Matroid,LiuChen94Matroid}\label{L:4}
Let $U$ be a set and $\mathbf{L}$ a family of subsets of  $U$. Then $\mathbf{L}$ is a family of close sets in a matroid if and only if $\mathbf{L}$ satisfies the following three conditions:\\
(F1) $U\in \mathbf{L}$;\\
(F2) If $F_{1}, F_{2}\in \mathbf{L}$, then $F_{1}\cap F_{2}\in \mathbf{L}$;\\
(F3) If $F\in \mathbf{L}$, and $\{F_{1}, F_{2},....,F_{k}\}$ is a family of minimal proper subsets containing $F$ in $\mathbf{L}$. Then $\{F_{1}-F, F_{2}-F,....,F_{k}-F\}$ is a partition of $U-F$.
\end{lemma}
\begin{proposition}
Let $R$ be an equivalence relation on $U$ and $M(R)$ the support matroid induced by $R$. Then the family of closed sets of $M(R)$ is $\mathbf{L}(R) = \{X\subseteq U:\bigcup\limits_{x\in X}{RN(x)} = X\}$.
\end{proposition}
\begin{proof}
(F1) is straightforward. We need to prove only that (F1) and (F2) of lemma~\ref{L:4}. For all $F_{1}, F_{2}\in \mathbf{L}(R)$,  $\bigcup \{RN(x)|x\in F_{1}\} = F_{1}$ and $\bigcup \{RN(x)|x\in F_{2}\}=F_{2}$. On one hand, since $\bigcup \{RN(x)|x$ $\in F_{1}\cap F_{2}\subseteq F_{1}\}\subseteq F_{1}$ and $\bigcup \{RN(x)|x\in F_{1}\cap F_{2}\subseteq F_{2}\}\subseteq F_{2}$, hence $\bigcup \{RN(x)|x\in F_{1}\cap F_{2}\}\subseteq F_{1}\cap F_{2}$. On the other hand, it is obvious that $F_{1}\cap F_{2}\subseteq \bigcup \{RN(x)|x\in F_{1}\cap F_{2}\}$. Therefore, $F_{1}\cap F_{2}\in \mathbf{L}(R)$. So $\mathbf{L}(R)$ satisfies F(2). Suppose $F_{i}=F\cup RN(x_{i})(i=1,2,...,k)$, then we have that $\{F_{1},F_{2},....,F_{k}\}$ is a family of minimal proper subsets containing $F$ in $\mathbf{L}(R)$ and $(F_{i}-F)\cap (F_{j}-F) = RN(x_{i})\cap RN(x_{j}) = \emptyset~(1\leq i\neq j\leq k)$. This, together with the fact that $\bigcup^{k}_{i=1}(F_{i}-F) = \bigcup^{k}_{i=1}RN(x_{i}) = U$, means that $F_{1}-F,F_{2}-F,....,F_{k}-F$ is a partition of $U-F$. So $\mathbf{L}(R)$ is satisfies (F3). Hence $\mathbf{L}(R)$ is the family of closed sets of $M(R)$.
\end{proof}

The following proposition shows the relationship between the support set induced by the intersection of two equivalence relations with two support set induced by these two equivalence relations respectively.
\begin{proposition}
Let $R_{1}, R_{2}$ be two equivalence relations on $U$. Let $M(R_{1})$\\, $M(R_{2})$ and $M(R_{1}\cap R_{2})$ be the support matroids induced by $R_{1}, R_{2}$ and $R_{1}\cap R_{2}$, respectively. Then $\mathbf{S}(R_{1}\cap R_{2})\subseteq \mathbf{S}(R_{1})\cap \mathbf{S}(R_{2})$.
\end{proposition}
\begin{proof}
  Suppose that $U/R_{1} = \{RN_{1}(x_{1}),RN_{1}(x_{2}),....,RN_{1}(x_{k})\}$, $U/R_{2} $ $= \{RN_{2}(x_{1}),RN_{2}(x_{2}),....,RN_{2}(x_{m})\}$ and $U/(R_{1}\cap R_{2}) = \{RN(x_{1}),$ $RN(x_{2}),....,RN(x_{n})\}(k\leq n, m\leq n)$ where $RN_{1}(x_{i}) = \{y\in U|x_{i}Ry\}$ $(1\leq i\leq k), RN_{2}(x_{i}) = \{y\in U|x_{i}Ry\}(1\leq i\leq m)$ and $RN(x_{i}) = \{y\in U|x_{i}Ry\}(1\leq i\leq n)$. According to Definition~\ref{D13}, for all $Y\in \mathbf{S}(R_{1}\cap R_{2})$, it is clear that $(R_{1}\cap R_{2})^{*}(Y) = U$. Thus for all $RN(x_{i})\in U/(R_{1}\cap R_{2})$, we have that $RN(x_{i})\cap Y\neq \emptyset$. Since $R_{1}\cap R_{2}\subseteq R_{1}$ and $ R_{1}\cap R_{2}\subseteq R_{2}$, we have that for all $RN(x_{i})\in U/(R_{1}\cap R_{2})(1\leq i\leq n)$, there exist $RN_{1}(x_{j})\in U/R_{1}(1\leq j\leq k)$ and $RN_{2}(x_{t})\in U/R_{2}(1\leq t\leq m)$ such that $RN(x_{i})\subseteq RN_{1}(x_{j})$ and $RN(x_{i})\subseteq RN_{2}(x_{t})$. Hence $RN_{1}(x_{i})\cap Y\neq \emptyset(1\leq i\leq k)$ and $RN_{2}(x_{i})\cap Y\neq \emptyset(1\leq i\leq m)$. According to Definition~\ref{D:2} and Definition~\ref{D13}, we have that $Y\in \mathbf{S}(R_{1})$ and $Y\in \mathbf{S}(R_{2})$. Therefore $\mathbf{S}(R_{1}\cap R_{2})\subseteq \mathbf{S}(R_{1})\cap \mathbf{S}(R_{2})$.
\end{proof}


\section{Conclusions}\label{sc4}
In this paper, we investigated the relationships between matroids and rough sets through support sets constructed by the upper approximation operator. First, we induce a matroid by an equivalence relation and a type of matroid, namely support matroid, is defined. Some characteristics of the matroid induced by equivalence relations, such as independent sets, bases,  hyperplanes, rank function and closed sets, have been well expressed by upper approximations and equivalence classes. Second, through closed sets,  we use matroidal approaches to describe precise sets in rough sets. Through the above work, we bridged matroids and rough sets. In future work, we will further connect rough sets and matroids from different aspects.
\section{Acknowledgments}

This work is supported in part by the National Natural Science Foundation of China under Grant No. 61170128, the Natural Science Foundation of Fujian Province, China, under Grant Nos. 2011J01374 and 2012J01294, and the Science and Technology Key Project of Fujian Province, China, under Grant No. 2012H0043.




\end{document}